\documentclass[12pt,letterpaper]{article}
\usepackage[left = 1 in, right = 1 in, top = 1 in, bottom = 1 in]{geometry}
\usepackage{multicol}
\usepackage{amsmath}
\usepackage{amssymb}
\usepackage{amsthm}
\usepackage{graphicx}
\usepackage{algorithm,algpseudocode}
\usepackage{setspace}
\usepackage{float}   
\usepackage{siunitx}
\usepackage{xspace}
\usepackage{xcolor}
\usepackage{soul}
\usepackage{upgreek}
\usepackage{listings} 

\usepackage{lineno}

\newtheorem{theorem}{Theorem}
\newtheorem{corollary}{Corollary}[theorem]

\definecolor{ForestGreen}{RGB}{46,111,64}

\usepackage{hyperref}
\usepackage[style=nature,hyperref=true]{biblatex}
\newcommand{\reffig}[1]{Fig.~\ref{#1}}
\newcommand{\reftab}[1]{Table~\ref{#1}}
\newcommand{\refsec}[1]{Section~\ref{#1}}

\newcommand{\refeq}[1]{Eq.~\ref{#1}}
\newcommand{\refcite}[1]{Ref.~\cite{#1}}

\newcommand{\loss}{\rho}
\newcommand{\Loss}{\text{C}}
\newcommand{\dTheta}{\Theta}
\newcommand{\dtheta}{\theta}
\newcommand{\dti}[1]{\theta_{#1}}
\newcommand{\dtit}[2]{\theta_{#1}(#2)}
\newcommand{\dtitn}[3]{\theta_{#1}^{#3}(#2)}
\newcommand{\dtin}[2]{\theta_{#1}^{#2}}
\newcommand{\gradC}{\mathbf{\nabla_\Theta} \Loss}
\newcommand{\gradCW}{\mathbf{\nabla_\W} \Loss}
\newcommand{\dC}{\Delta \Loss}
\newcommand{\dCi}[1]{\frac{\partial \Loss}{\partial \theta_{#1}}}
\newcommand{\dCwi}[1]{\frac{\partial \Loss}{\partial w_{#1}}}
\newcommand{\E}[1]{\mathrm{E}\Big(#1\Big)}
\newcommand{\cov}{\mathrm{Cov}}

\newcommand{\xhat}[0]{\ensuremath{\hat{x}}\xspace}
\newcommand{\yhat}[0]{\ensuremath{\hat{y}}\xspace}
\newcommand{\tautheta}{\ensuremath{\tau_\uptheta}\xspace}

\newcommand{\W}{{\bf W}}
\newcommand{\bb}{{\bf b}}

\newcommand{\xx}{{\bf x}}
\newcommand{\yy}{{\bf y}}
\newcommand{\mm}{{\bf m}}
\newcommand{\vv}{{\bf v}}
\newcommand{\yyHat}{\hat{\yy}}
\newcommand{\Zero}{{{\bf 0}}}
\newcommand{\Gmgd}{{\bf G}\xspace}
\newcommand{\Gtrue}{\ensuremath{\gradCW}\xspace}
\newcommand{\LL}{\boldsymbol{\Lambda}}
\newcommand{\s}[1]{\sigma_{#1}}
\renewcommand{\a}[1]{\alpha_{#1}}

\newcommand{\nP}{\ensuremath{K}\xspace}  
\newcommand{\nW}{\ensuremath{N}\xspace}  
\newcommand{\nT}{\ensuremath{\tau_\theta}\xspace} 

\begin{document}

\doublespacing
\pagenumbering{arabic} 


\title{Scaling of hardware-compatible perturbative training algorithms}

\author{
B. G. Oripov$^{1,2}$,
A. Dienstfrey$^{1}$,\\
A. N. McCaughan$^1$,
S. M. Buckley$^{1,*}$
}

\date{
     \small
     $^1$National Institute of Standards and Technology, Boulder, Colorado 80305, USA\\%
     $^2$Department of Physics, University of Colorado, Boulder, Colorado 80309, USA\\%
     $^*$Correspondence: sonia.buckley@nist.gov
 }
\maketitle

\section*{Abstract}
In this work, we explore the capabilities of multiplexed gradient descent (MGD), a scalable and efficient perturbative zeroth-order training method for estimating the gradient of a loss function in hardware and training it via stochastic gradient descent. We extend the framework to include both weight and node perturbation, and discuss the advantages and disadvantages of each approach. We investigate the time to train networks using MGD as a function of network size and task complexity. Previous research has suggested that perturbative training methods do not scale well to large problems, since in these methods the time to estimate the gradient scales linearly with the number of network parameters.  However, in this work we show that the time to reach a target accuracy--that is, actually solve the problem of interest--does not follow this undesirable linear scaling, and in fact often decreases with network size. Furthermore, we demonstrate that MGD can be used to calculate a drop-in replacement for the gradient in stochastic gradient descent, and therefore optimization accelerators such as momentum can be used alongside MGD, ensuring compatibility with existing machine learning practices. Our results indicate that MGD can efficiently train large networks on hardware, achieving accuracy comparable to backpropagation, thus presenting a practical solution for future neuromorphic computing systems. 

\section{Introduction}

Machine learning (ML) algorithms are fundamentally altering our interactions with technology, driven predominantly by artificial neural networks. However, the significant costs associated with training and deploying these algorithms—largely stemming from energy expenditures—pose a substantial barrier to their scalability and broader adoption. Industry leaders have voiced concerns over the energy demands of conventional complementary metal-oxide semiconductor (CMOS) hardware used in ML, advocating for substantial investments in innovative hardware solutions \cite{wsj_article}.
In contrast to this, the human brain achieves similar computational feats at a fraction of the energy cost, suggesting that brain-inspired hardware represents a promising direction. In this context, analog neuromorphic hardware offers a promising solution to the energy challenges being faced by current hardware technologies. However, training on analog hardware has proved more difficult than on its digital counterparts.
In this article, we demonstrate the effectiveness of multiplexed gradient descent (MGD), a general perturbative training framework, in matching the accuracy of the backpropagation algorithm in identical network architectures, even for large networks ($>10^6$ parameters). We also address common misconceptions regarding perturbative training methods and show that they can indeed scale, contrary to prevailing sentiment in the field \cite{hinton2022forwardforward, Lillicrap2020}.

Designing dedicated neuromorphic hardware introduces several challenges, particularly when considering the methods for training such systems. Training a ML model amounts to minimizing a specified loss function. In ML contexts this minimization proceeds by gradient descent and its various extensions. In nearly all standard computing frameworks, this gradient is computed by reverse-mode automatic differentiation and is referred to as backpropagation. While this algorithm is efficient in software, implementing backpropagation in hardware poses significant difficulties including the requirement that the in-hardware computational path can be reversed, substantial memory at each neuron, and the necessity of computing the derivative of the activation function. Due to these challenges, implementation of backpropagation in analog hardware to date \cite{Li2018, Pai2023, Nandakumar2020} has typically involved a computer in the loop to implement part of the computation, or been limited to relatively small networks \cite{Vandoremaele2024}. To avoid the difficulties of a full hardware implementation of backpropagation, the field has explored several alternative training approaches.

One approach is to conduct training in a traditional computer using a model of the hardware, transfer the resulting weights into devices, and restrict the in-hardware computation to inference tasks only. This solution can reduce deployment costs in many cases, as a single training simulation can determine parameters for many inference instantiations. However, discrepancies between the simulated hardware model and the actual hardware can result in diminished accuracy. Developing training algorithms that produce networks that are more robust to device-to-device variations is a significant line of research \cite{Rasch2023}. Additionally, this approach is less viable for situations that require in-situ adaptability. Various "computer-in-the-loop" strategies \cite{Buckley2023} can help with these issues, for example by implementing the forward pass in hardware and calculating the individual weight updates via simulation \cite{Wright2022}.

Another strategy employs Hebbian learning, a simpler learning paradigm based on the empirical observation that the connections between neurons strengthen when they fire synchronously. Although straightforward to implement in hardware \cite{Bichler2012, Friedmann2017}, Hebbian learning lacks the general applicability and mathematical rigor of gradient-based methods and does not guarantee convergence to a solution. Recent insights suggest that the brain's learning mechanisms might involve more complex three-factor rules \cite{Gerstner2018}, indicating that Hebbian learning might not fully capture the neural learning processes.

In contrast, perturbative methods offer a model-free \cite{Dembo1990} stochastic gradient-descent approach, treating the hardware as a black box, and applying small perturbations to estimate the gradient and therefore minimize network cost using the same optimization algorithm as traditional ML. This approach does not rely on a model of the network's operation, allowing it to be applied across various hardware platforms. Despite early interest \cite{Dembo1990, Matsumoto1990, Cauwenberghs1992, Flower1992,  Alspector1992, Kirk1992, Maeda1995, Cauwenberghs1996,Moerland1996, Montalvo1997,Miyao1997, Draghici2000} and ease of implementation, perturbative techniques fell out of favor due to the poor scaling of time to estimate the gradient. Critically, gradient estimation time was assumed to be a good proxy for training time \cite{Lillicrap2020, hinton2022forwardforward, Cauwenberghs1994}. In this paper we test this assumption and find that the connection between gradient and training accuracy is much more nuanced. Our findings are bolstered by a number recent papers that have also found perturbative techniques to be more effective than assumed \cite{mgd_paper, Dalm2023, ren2023scaling, Zhao2023} and are likewise convenient to implement in hardware \cite{Bandyopadhyay2022}.

Multiplexed gradient descent (MGD) is a model-free gradient descent framework that attempts to generalize earlier perturbative approaches\cite{Dembo1990}\cite{Spall1992} in a hardware-friendly way by defining a set of three time constants for the perturbation process that correspond to the time between weight updates, time between sample updates, and time between perturbation updates.  By varying these time constants in a given hardware system, a wide variety of numerical gradient descent techniques (e.g. coordinate descent, SPSA, etc) can be achieved. In previous work \cite{mgd_paper}, we introduced MGD and examined the speed of gradient estimation and training time of different perturbative techniques given particular hardware parameters. For modest-sized networks and limited architectures, the results indicated that in-hardware training with MGD could be performed faster than backpropagation on a standard GPU. 

In this study, we extend the application of MGD to train larger and deep feedforward neural networks. We show that perturbative techniques can match the accuracy of backpropagation for networks of up to one million parameters. We evaluate the scaling of both the time to estimate the gradient and the time to train to a given accuracy as a function of network size. As expected, we find that the speed of gradient convergence can be estimated analytically as a function of network size. However, the time to attain a converged approximation of the gradient is not representative of the time to train a network. This result agrees with previous work showing that an accurate gradient estimate is not necessarily required for online gradient descent  \cite{Lillicrap2016}. Furthermore, in this study we perform a more systematic analysis of different perturbative approaches--- weight perturbation \cite{Dembo1990} and node perturbation \cite{Flower1992}--- and discuss the tradeoffs associated with their implementations. For example, while the time to estimate the gradient does scale better for dense networks with node perturbation compared to weight perturbation, this is not true for arbitrary architectures \cite{Zuge2023}--as one example, convolutional networks may have more nodes than weights. Finally, we show that standard optimization algorithms such as momentum and the Adam optimizer can be implemented in the MGD framework. This is important, as in practice most ML implementations rely on more than vanilla backpropagation, and therefore it is likely that more advanced optimization techniques will be required in the MGD framework as well. This is supported by recent results \cite{ren2023scaling}. Together, these results suggest that MGD represents a promising class of algorithms for training new emerging neuromorphic hardware. 

\section{Multiplexed gradient descent}
\label{sec:mgd}
Multiplexed gradient descent provides an approach for model-free, {\em in situ} training of hardware-based implementations of neural networks. In this section we expand upon the ideas originally presented in \refcite{mgd_paper}. We first generalize the connection between perturbation parameters and neural network weights. In \cite{mgd_paper} these two variable classes were the same--that is, every individual weight had its own perturbation--resulting in what is typically called the ``weight perturbation'' model. This model was investigated exclusively in \refcite{mgd_paper}. By relieving this constraint, we are able to implement a ``node perturbation'' model.  Below, we define node perturbation and study its impact on training efficiency. Other perturbation models are also possible although not investigated here. We next recast the MGD algorithm in a more formal framework. As a result we prove a theorem stating that, for a linear cost function or, equivalently, for sufficiently small perturbation magnitudes, the principal random variable defined in the MGD formalism is an unbiased estimator of the true parameter gradient. We also compute its second order statistics. Finally, we introduce the distinction between gradient convergence and network accuracy. Whereas the former has generally been the focal point for prior studies in perturbative training of neural networks \cite{Cauwenberghs1994, Lillicrap2020}, our numerical studies in the following sections call into question whether this attention is fully warranted. 

\subsection{Weight and Node Perturbation}
To present the MGD idea in its most general form, we distinguish two classes of parameters. At the hardware level one considers the set of {\em perturbation parameters}
\[
\Theta=(\theta_1,\theta_2,\ldots,\theta_\nP).
\]
These are the physical quantities---for example, voltages, currents, conductances, optical power---that can be modified most readily by the hardware engineer.  The precise definition depends on the physics and interface to the hardware-based model implementing the neural network, and can change from one design to the next. For the neural network, one considers the usual mathematical model of a parameterized non-linear function. The parameters are referred to as ``weights'' and ``biases'' which we combine into a large vector  which we will simply refer to as the weights,
\[
\W=(w_1,\ldots,w_\nW).
\]

The hardware-based implementation of the neural network takes input $\xx$ to output $\yy$. The map depends on both vectors
\[
\yy = f(\xx;\W,\Theta)
\] 
We consider a supervised learning problem in which we are provided with labeled training data $\{(\xx_1,\yy_1),\ldots,(\xx_T,\yy_T)\}$. For a given input $\xx_t$, we use $\yyHat_t=f(\xx_t;\W;\Theta)$ to denote the associated inference. We assume a hardware implementation of a differentiable residual function that evaluates the quality of the network inference in comparison to the target, $\loss(\yyHat_t,\yy_t)$. We also assume a hardware-based accumulator that computes the empirical loss as the mean of the residual over some collection of training data \footnote{Note, whereas details of normalization factors are important to make mathematical derivations precise, in practice there is often flexibility that can be accommodated by a well-placed scale factor. For example, loss function scaling can be passed into gradient scaling both of which, operationally, can be absorbed into the setting of the learning rate.}
\begin{equation}
\label{def:Loss}
    \Loss(\W,\Theta) := \frac{1}{T}\sum_{t=1}^T\loss(f(\xx_t;\W,\Theta),\yy_t)
\end{equation}
Here the sum could be over the full set of training data, or some minibatch. The training problem consists of minimizing the loss as a function of $\W$. Within the ML community variations of gradient descent are nearly universally used for this task. In all cases one requires an estimate of the gradient
\begin{equation}
\label{def:gradW}
    \gradCW := \left(\dCwi{1},\dCwi{2},\ldots,\dCwi{\nW}\right).
\end{equation}
In computational networks backpropagation is generally used to evaluate the above gradient.

The analytic details of backpropagation do not readily lend themselves to hardware based implementation. By contrast, the premise of hardware-based ML entails that there are parameters whose variation have measurable impact on network performance. MGD amounts to a strategy for organizing these perturbative degrees of freedom so as to estimate
\begin{equation}
\label{def:gradC}
\gradC := \left(\dCi{1},\dCi{2},\ldots,\dCi{\nP}\right).
\end{equation}
The central assumption is that in a neighborhood of a fixed parameter vector, $\|\W-\W_0\|<\delta$, variation in $f(\xx;\W,\Zero)$ as a function of $\W$ can be inferred from variation in $f(\xx;\W_0,\Theta)$ as a function of $\Theta$. More specifically, we assume a relationship between gradients
\[
\gradC \leftrightarrow \gradCW.
\]
This connection is used to estimate the gradient of a cost function with respect to the weights after which training proceeds by gradient descent and its variants. 

In the simplest setting, the weights themselves are directly perturbable,
\begin{equation}
\label{def:weightPerturbation}
    \W = \W_0+\Theta.
\end{equation}
In this case, the two gradients \eqref{def:gradW} and \eqref{def:gradC} are one and the same. This is referred to as {\em weight perturbation.} A different strategy relies on the typical construction of neural network functions as compositions of an affine transformation followed by a non-linear activation function. In this case one can consider a perturbative degree of freedom added to the output of the affine transformation\footnote{The non-linear activation function is assumed to act componentwise on each of its inputs.}
\begin{equation}
\label{def:nodePerturbation}
\xx^{(\ell+1)} 
    = \sigma\big(\W^{(\ell)}\xx^{(\ell)}+\bb^{(\ell)} + \Theta\big),\text{ for }\Theta = {\bf 0}.
\end{equation}
For any fixed training instance we abbreviate the loss as $\loss_t=\loss(f(\xx_t;\W,\Theta),\yy_t)$. The chain rule then implies that
\begin{subequations}
\label{test}
\begin{align}
\frac{\partial \loss_t}{\partial b_k} &= \frac{\partial \loss_t}{\partial \theta_k} \\
\label{eqn:nodeMultiply}
\frac{\partial \loss_t}{\partial w_{k,j}} &= \frac{\partial \loss_t}{\partial \theta_k}x^{(\ell)}_j.
\end{align}
\end{subequations}
Summing over all training instances in \eqref{def:Loss} gives the gradient of the cost. As the perturbation parameter in \eqref{def:nodePerturbation} varies the input into an activation function and the latter are depicted as nodes in the computational graph, this is referred to as {\em node perturbation}. Note that a fully hardware-based implementation of node perturbation requires a multiplication circuit to infer $\gradCW$ from $\gradC$ as shown in Eq. \eqref{eqn:nodeMultiply}.

\begin{figure}
    \centering
    \includegraphics[width=4.5in]{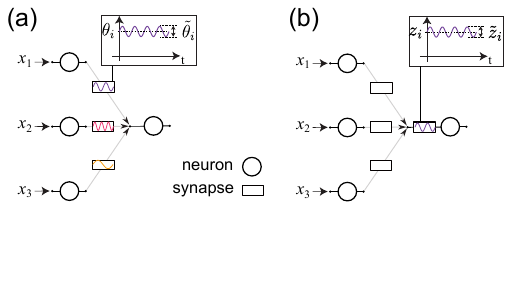}
    \caption{(a) Illustration of weight perturbation, where each weight is perturbed individually. Weights can be updated entirely within the synapse by combining the global cost-feedback signal with the local perturbation.  (b) Illustration of node perturbation, where only the summed input to each neuron is perturbed.  Here, weight updates must be performed through a one-step backpropagation process that computes the error at the neuron input, and then passes it backwards (through a multiplication) to the synapse.}
    \label{fig2-WvsN}
\end{figure}

These two perturbation strategies are represented graphically in \reffig{fig2-WvsN}. Figure~\ref{fig2-WvsN}(a) depicts the weight perturbation algorithm, where individual weights are perturbed and then correlated with the corresponding change in the cost. As we will see below, perturbation gradient calculation time scales linearly with $\text{dim}(\Theta)=\nP$. In the weight perturbation model, $\nP=\nW$. In \refcite{mgd_paper}, we showed in detail how weight perturbation could be implemented in hardware and can be used for emergent learning in a neuromorphic hardware system. However, there has long been concern that the scaling of this algorithm is too slow for larger modern datasets.  Figure~\ref{fig2-WvsN}(b) depicts the node perturbation algorithm \cite{Flower1992}, where the perturbation is instead applied to the input to the activation function, and a single-layer backpropagation is performed to compute the weight update. In this case, the gradient calculation time scales with the number of nodes in the network. This is significant for dense layers: where the number of weights scale as $\nW$, and the number of activations or nodes scale as $\nP=\sqrt{N}$. However, for different layer types, this may not be the case. For example, in convolutional layers, there may be more activations than weights. This can be seen in \reftab{numparams}, where for example layers 1 and 2 have more activations than weights. The ubiquity of convolutional layers in modern machine learning motivates us to compare the two approaches in terms of speed and then their implementation in hardware. We note that single-layer backpropagation is still a completely local learning rule - signals still do not need to propagate backwards through weights or activations. 

\begin{table}[h]
\centering
\begin{tabular}{|c|c|c|c|c|}
\hline
Layer & Layer Type     & kernel & num. parameters & num. activations \\
\hline
1     & Convolutional  & 3 $\times$ 3 $\times$ 1 $\times$ $d$ & 480             & 37632            \\
2     & Convolutional  & 3 $\times$ 3 $\times$ $d$ $\times$ $d$ & 20784           & 37632            \\
3     & MaxPool  & 2 $\times$ 2  & -           & -            \\

4     & Convolutional  & 3 $\times$ 3 $\times$ $d$ $\times$ 2$d$ & 41568           & 18816            \\
5     & Convolutional  & 3 $\times$ 3 $\times$ 2$d$ $\times$ 2$d$ & 83040           & 18816            \\
6     & MaxPool  & 2 $\times$ 2  & -           & -            \\
7     & Convolutional  & 3 $\times$ 3 $\times$ 2$d$ $\times$ 4$d$ & 166080          & 9408             \\
8     & Convolutional  & 3 $\times$ 3 $\times$ 4$d$ $\times$ 4$d$ & 331968          & 9408             \\
9     & MaxPool  & 2 $\times$ 2  & -           & -            \\

10     & Dense          & 36$d$ $\times$ 4$d$                      & 331968          & 192              \\
11     & Dense          & 4$d$ $\times$ 4$d$                         & 37056           & 192              \\
12     & Dense          & 4$d$ $\times$ 10                          & 1930            & 10               \\
\hline
Total &                &                                            & 1014874         & 132106           \\
\hline
\end{tabular}
\caption{Number of parameters (weights) and activations (nodes) for each layer to be trained in a medium sized network evaluated in this work. In this specific example the depth of the input convolutional layer is $d=48$.}
\label{numparams}
\end{table}

\subsection{Perturbative gradient estimation}
We introduce perturbations as a function of time
\begin{equation}
\label{def:dTheta}
\dTheta(t):= \big(\dtit{1}{t},\dtit{2}{t},\ldots,\dtit{\nP}{t}\big)
\end{equation}
Many different types of perturbations can be implemented \cite{Dembo1990, mgd_paper}. For example, components of $\dTheta$ can be assigned distinct frequencies, or they can be associated with elements of an orthogonal code. In this work, we used the later, Bernoulli perturbations where simultaneous discrete perturbations of $\dtheta_i(t)=\pm\delta$ for every parameter every timestep. We consider time to be discrete, $t=1,\ldots,\nT,$ and each parameter is associated with an independent random variable. Thus, \eqref{def:dTheta} represents a collection of $\nP\cdot \nT$ random variables in total. We assume that all random variables are independent, and that the distributions are fixed over time and are mean zero.

It follows that the covariance matrix is diagonal and constant over time
\begin{equation}
\label{eqn:thetaCov}
    \LL:=\E{\dTheta(t)\dTheta(t)^T} = \E{\dTheta(t')\dTheta(t')^T}={\rm diag}(\s{1}^2,\ldots,\s{\nP}^2).
\end{equation}

Multiplexed gradient descent uses a correlation analysis of the perturbed loss function to provide an estimate of its gradient. By Taylor expansion we have (the dependence on $\W$ is omitted for readability)
\begin{equation}
\dC(\dTheta(t)) := C(\dTheta(t)) - C(\Zero) = \gradC\cdot\dTheta(t) +\mathcal{O}(\dtit{i}{t}\dtit{j}{t}).
\end{equation}
If we assume that the perturbations are characterized by a small scale, $\max_{i,t} {\rm E}(\dtitn{i}{t}{2}) \leq \epsilon$,
then for $\epsilon$ sufficiently small, one may approximate $\dC$ by
\begin{equation}
\label{eqn:linearCost}
\dC(t) 
 \approx \gradC\cdot\dTheta =  \sum_{i=1}^{\nP} \dCi{i}\dtit{i}{t}.
\end{equation}
Finally, we define the MGD random vector $\Gmgd$ as
\begin{align}
\label{def:G}
\Gmgd &:= \frac{1}{\nT}\LL^{-1}\sum_{t=1}^{\nT} \dC(t)\dTheta(t), \mbox{ which in component form is}\\
  G_m &= 
  \frac{1}{\nT\s{m}^2}\sum_{t=1}^{\nT}\sum_{\mu=1}^{\nP} \dCi{\mu}\dtit{\mu}{t}\dtit{m}{t}
\end{align}
where we recall the covariance matrix of $\Theta$ defined in \refeq{eqn:thetaCov}. We claim that the random vector $\Gmgd$ can be used as an approximation to $\gradC$ in the optimization routines that underpin machine learning. In the appendix we show that $\Gmgd$ is an unbiased estimator of the gradient, ${\rm E}(\Gmgd)=\gradC$, and we compute its full covariance matrix, $\cov(\Gmgd)$. The key result is that this $\cov(\Gmgd)$ tends to zero as $\nT\to\infty$. In other words, averaging over a long time-period yields an increasingly deterministic approximation to $\gradC$. These technical points motivate the following numerical experiments. 

\begin{figure}[H] 
    \centering
    \includegraphics[width=3.5in]{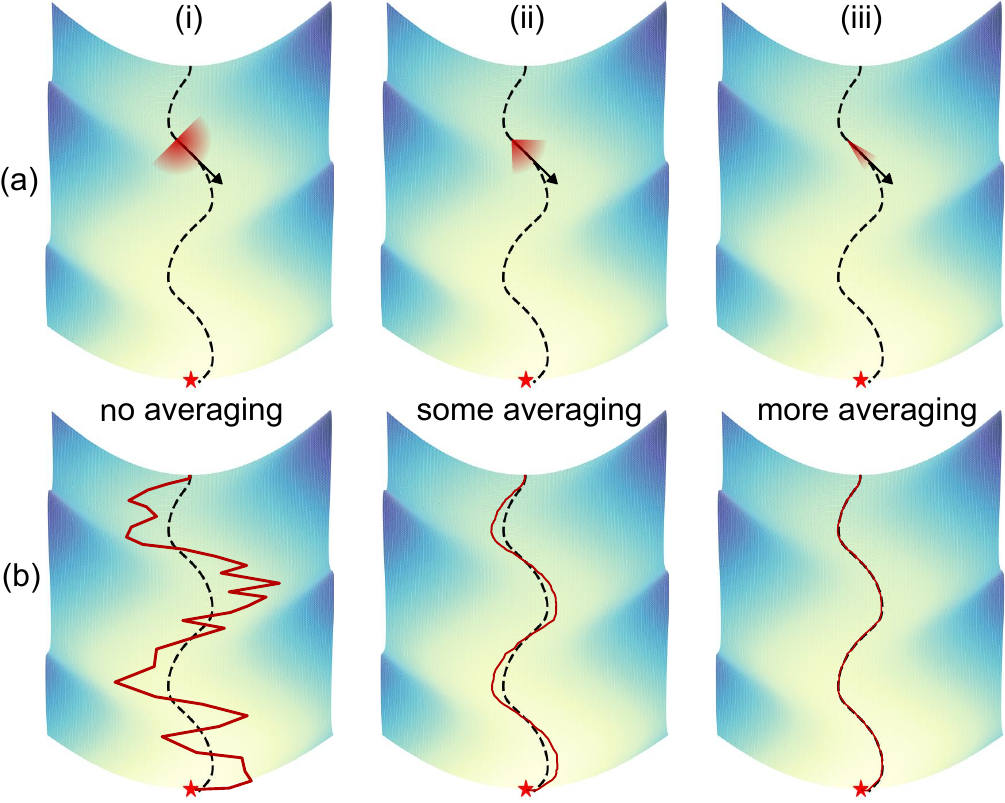}
    \caption{Illustration of gradient descent down a cost landscape using MGD, for varying amounts of averaging. (a) In MGD, the true gradient (black dashed line) is not presumed to be known analytically, and is instead estimated. The gradient estimation (red-shaded regions) starts out inaccurate but can be refined to arbitrary precision with additional averaging toward the true gradient.  (b) The resulting weight updates and trajectory down the cost-landscape (red lines) do not strictly follow the gradient, but instead deviate to a degree that depends on the amount of averaging.}
    \label{fig1-cartoon}
\end{figure}

\subsection{Network accuracy versus gradient accuracy}
\reffig{fig1-cartoon} presents a schematic representation of a cost landscape, and a illustrative comparison of descent curves using MGD versus backpropagation. The black arrow in \reffig{fig1-cartoon}(a) represents the direction of the true gradient as calculated by backpropagation (\Gtrue), while the black dashed line in both (a) and (b) represents the steepest descent path that the backpropagation algorithm would follow in the cost landscape. For comparison, the possible directions for the gradient estimate from MGD (\Gmgd) are represented by the red shading around the arrow, while the red solid line shown in (b) illustrates the path taken by the MGD algorithm. It can be seen from the figure that the MGD algorithm tends to follow the general direction of gradient descent backpropagation, with additional stochasticity. 

The level of stochasticity is determined by the gradient-integration time constant, \tautheta, and has been discussed in detail in \refcite{mgd_paper}. For short gradient integration time \tautheta (e.g. \reffig{fig1-cartoon} (i)), the MGD algorithm computes a directional derivative along a random direction, and the descent is highly stochastic. For longer gradient integration times \tautheta (e.g. \reffig{fig1-cartoon}  (ii) and (iii)), the gradient estimate \Gmgd calculated by MGD converges to the direction of \Gtrue. In this case, the MGD algorithm replicates the behavior of the backpropagation. Note that no matter what the value of \tautheta, \Gmgd is always pointing ‘downhill’, although it may not be pointed in the direction of steepest descent. 

In \refsec{sec:simulations}, we demonstrate that a network can be trained to the same accuracy as backpropagation even without the need for averaging (large \tautheta) at each step. We will also discuss the scaling of the gradient estimation time and the training time with network size. These two convergences are illustrated, respectively, across the top and bottom rows of \reffig{fig1-cartoon}. We define the `gradient estimation time' as the time it would take for the cosine of the angle between \Gmgd and \Gtrue to be $>$ 0.95. The ‘training time’ is the time taken from initialization to reach a particular cost or accuracy value.

\section{Scaling of the MGD algorithm training speed}
\label{sec:simulations}
In the following simulations we compare the performance of the MGD algorithm in terms of training time for networks of different sizes using both weight and node perturbation, and compare it to the gradient estimation time. The learning rates for weight and node perturbation were kept approximately equivalent using the normalization calculations in Appendix A.
Using the MGD time constants, we can perform quantitative comparisons of the speed of training these networks. The network we used throughout the paper consisted of 6 convolution neural network layers followed by 3 dense neural network layers with 3 max-pool layers in between (see \reftab{numparams}). The network was trained to classify images in the FashionMNIST dataset. We kept this architecture constant but we varied the size of the layers ($d=$ 1, 2, 4, 8, 16, 32, 48, 64) to change overall network size. \reftab{numparams} shows an example of such a network with $\sim$1 million parameters ($d=$ 48). The algorithm used for network training is described in Appendix A.

\begin{figure}[H] 
    \centering
    \includegraphics[width=\textwidth]{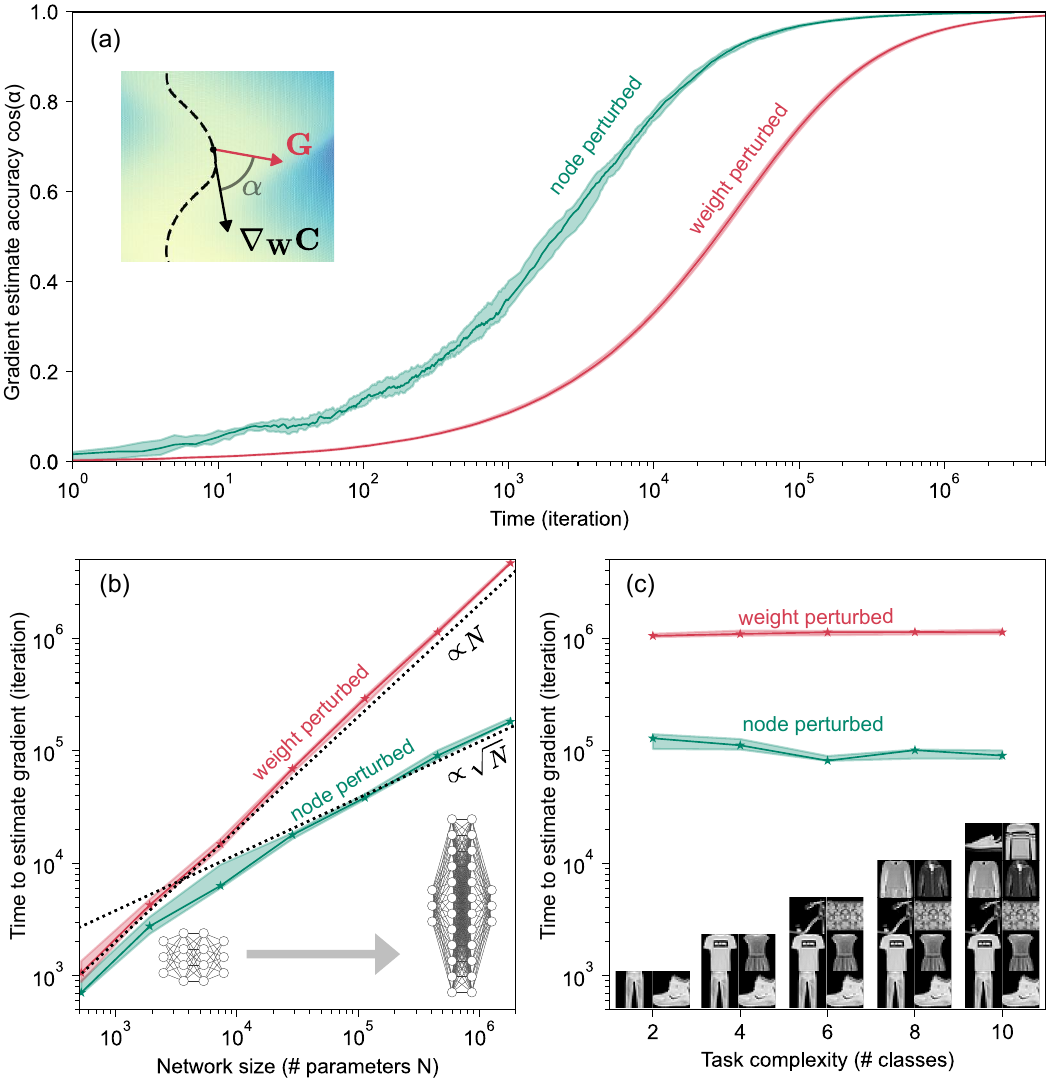}
    \caption{Analyzing the behavior of gradient estimation in MGD. (a) The accuracy of the gradient estimation versus time. Each iteration refines the the accuracy of the gradient estimation, defined as the cosine of the angle ($\alpha$) between the true local gradient $\gradCW$ and local gradient estimated by MGD $\Gmgd$ for a network with \nW$=4.52\times10^5$ parameters. (b) Number of iterations it takes for $\cos(\alpha)$ to reach 0.95 versus total number of trained parameters in the network (\nW). Dashed lines represent the expected \nW and $\sqrt{\nW}$ scaling for weight nad node perturbation, respectively. (c) Number of iterations it takes for $\cos(\alpha)$ to reach 0.95 (or 95\% gradient estimation accuracy) versus the task complexity. The shaded regions represent the upper and lower quartile bounds of 10 random initializations, and the solid line corresponds to the median.}
    \label{fig3-angleconv}
\end{figure}

\subsection{Gradient estimation}
We first investigate the time to accurately estimate the gradient for the architecture described in the previous section as a function of the number of network parameters. To measure the gradient estimation accuracy as a function of time, we begin by randomly initializing the network.  Then, for each iteration, we generate a new gradient estimate in the MGD fashion by randomly perturbing the network, measuring the resulting change in cost, and multiplying them together to get \Gmgd as in \eqref{def:G}. Each iterative estimation of \Gmgd is summed together and the total is compared with the true local gradient\Gtrue (computed by backpropagation), generating the plots shown in \reffig{fig3-angleconv}a. This was performed for both weight perturbation (red lines) and node perturbation (green lines). It is clear in both cases that when averaged for a sufficiently large number of iterations, \Gmgd becomes a perfect proxy for \Gtrue (matching both direction and amplitude of \Gtrue as shown in Appendix A). As evident from this figure, node perturbation can estimate \Gtrue much faster than weight perturbation due to the smaller number of independent perturbations involved. This simulation was repeated for multiple networks with varying number of parameters, where the depths of the network, hence number of layers was kept constant, and number of parameters in each layer was proportionally scaled. \reffig{fig3-angleconv}b shows the number of iterations required for the estimated gradient \Gmgd to be approximately aligned with the true gradient ($\cos{\alpha}=0.95$) as a function of network size. It is clear that, as expected, the time required to accurately estimate \Gtrue using MGD scales with the number of total parameters trained in the case of weight perturbation, and with the square root of the number of total parameters trained in case of node perturbation. We additionally performed a similar set of simulations varying task complexity for a fixed network size of $4.52\times10^5$ parameters. To vary task complexity we simply used the subset of data in FashionMNIST dataset corresponding to 2, 4, 6, 8 or 10 classes. Here, a less complex task means that there are less classes to classify images into. As expected and is evident from \reffig{fig3-angleconv}c task complexity has no effect on gradient estimation.  These results are in keeping with previous results using perturbative techniques.

\begin{figure}[H] 
    \centering
    \includegraphics[width=\textwidth]{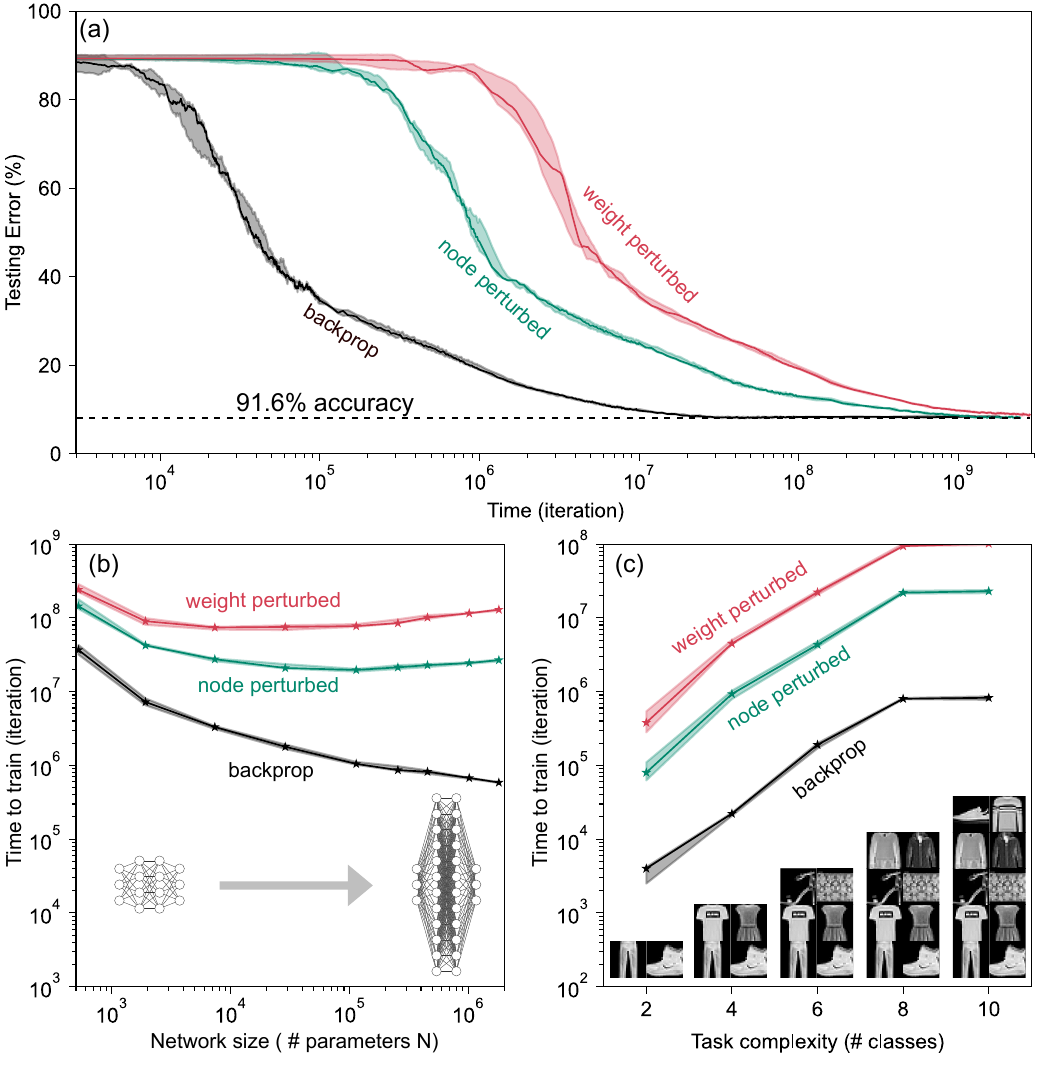}
    \caption{Time needed to successfully train networks with MGD as a function of network size and task complexity (a) Testing error versus number of iterations for network with $\nW=2.55\times 10^4$ parameters, showing that the MGD algorithm can match the same final accuracy of 91.6\% as backpropagation. Although backpropagation is generally not available in hardware, here in simulation its analytical nature provides the fastest time-to-train as expected. (b) Number of iterations it takes for the network to be trained to 80\% testing accuracy versus total number of trained parameters in the network N. (c) Number of iterations it takes for the network to be trained to 80\% testing accuracy versus the task complexity.
}
    \label{fig4-accuracy}
\end{figure}

\subsection{Image classification}
While \reffig{fig3-angleconv} shows that the scaling of the gradient follows the expected trends, we note that this is not the important figure of merit for ML tasks--ultimately, computing the gradient is only a means towards training the network. We next investigated whether the same scaling applies to the time to train a network as to the time to estimate the gradient. This scaling cannot be easily predicted analytically. However, in real-world applications, peak accuracy and time to reach that peak accuracy are typically the relevant performance metrics. This has not been previously investigated in detail for perturbative algorithms. \reffig{fig4-accuracy}(a) shows testing error versus number of iterations for a network with $2.55\times 10^5$ parameters trained to classify all 10 classes of FashionMNIST dataset for weight perturbation (red) and node perturbation (green) and backpropagation (black). As before, an iteration on the x-axis represents a single perturbation and subsequent gradient estimation of the network. To do a comparison on the scale with backpropagation (black line in \reffig{fig4-accuracy}), we perform the same simulation but use the true gradient instead of \Gmgd for the weight update. In this demonstration we set the gradient integration time \tautheta to 1, hence weights are updated after each perturbation with no gradient averaging. This means that we are moving along a direction that is always downhill (directional derivative), but a poor gradient estimate (see \reffig{fig1-cartoon}(i)). We can see that both weight and node perturbation can achieve the same accuracy as backpropagation on the same network architecture.

This simulation was then repeated for the same set of network sizes as in \reffig{fig3-angleconv}(b). \reffig{fig4-accuracy}(b) shows the number of iterations required for the network to reach 80\% accuracy versus the total number of parameters in the network. The time to reach 80\% accuracy varies by less than one order of magnitude over $>3$ orders of magnitude change in network size for both weight and node perturbation. This is markedly different from the gradient estimation time in \reffig{fig3-angleconv}(b), which in the case of weight perturbation scales proportionally to network size. This result is in contrast to some of the scaling arguments that have been previously made about perturbative techniques \cite{hinton2022forwardforward, Lillicrap2020}.

In addition, while node perturbation reduces the time required to reach a given accuracy when compared to weight perturbation, the performance enhancement depends on the required accuracy, and is not a simple relationship as for the gradient estimation time. For example, in \reffig{fig4-accuracy} b) we observe that the time to reach 80\% accuracy using node perturbation was reduced by approximately a factor of 2 over node perturbation for most network sizes.  However, we can see from the individual cost versus training time curve in \reffig{fig4-accuracy}(a) that this scaling factor was different depending on the desired final accuracy. For example, node perturbation is 10$\times$ faster to get to 50\%. As we did in the previous section, the same comparison was then performed for varying task complexities. \reffig{fig4-accuracy}(c) shows that task complexity strongly affects the time required to reach 80\% accuracy. The task complexity is clearly very important to the training time and can supersede network size in affecting training time.

These results demonstrate conclusively that a network with more than 1 million parameters can be trained to the same testing accuracy as backpropagation using perturbative zero-order optimization techniques. Additionally, the results indicate that while the time required to accurately estimate \Gtrue is indeed proportional to the number of trained parameters in the network, time required to reach a set accuracy target is a more complicated function of network size. This result is in keeping with previous observations that an accurate gradient estimate is not required for machine learning. While MGD is slower when compared to backpropagation when simulated on a digital computer, iterations could be implemented very quickly on a dedicated analog hardware.  To get an intuitive sense of this, consider that $10^9$ MGD time iterations would take 16 minutes on hardware with a modest speed of 1 MHz for perturbations, inference and updates.  Another interesting observation is that although there is a speedup associated with node perturbation, in general we find that node perturbation does not perform as well as expected compared to weight perturbation. For example, node perturbation is only 2$\times$ faster than weight perturbation to get to 80\% accuracy for a network of 1 million parameters. This indicates that learning rate and other hyperparameter factors may be as or more important than node versus weight perturbation in determining time to solution. Additionally, convolutional layers can sometimes have fewer weights than activations, invalidating any scaling advantage to node perturbation. In general, we found that it is more difficult to optimize node perturbation than weight perturbation, suggesting that node perturbation may be more sensitive to hyperparameter choices.

\section{Tailoring MGD to specific hardware}
MGD allows training to be optimized for specific hardware platforms. For example, some hardware platforms are based on non-volatile memory technologies with slow update speeds and a limited number of write-erase cycles before the memory elements begin to deteriorate. This type of hardware can be accommodated by adjusting the MGD time constants to reduce the number of weight updates. All of the examples we described above were calculated for $\tautheta = 1$, meaning the gradient was estimated for a single time step and the network was updated immediately afterwards. Increasing this integration time such that $\tautheta \gg 1$ can give similar performance with fewer overall weight updates. In a practical implementation, this could be accomplished if perturbations are implemented separately and faster than weight updates, for instance by placing a fast, volatile, perturbation element (such as a transistor) in series with a the slower-to-update and less-durable weight.

\begin{figure}[H] 
    \centering
    \includegraphics[width=\textwidth]{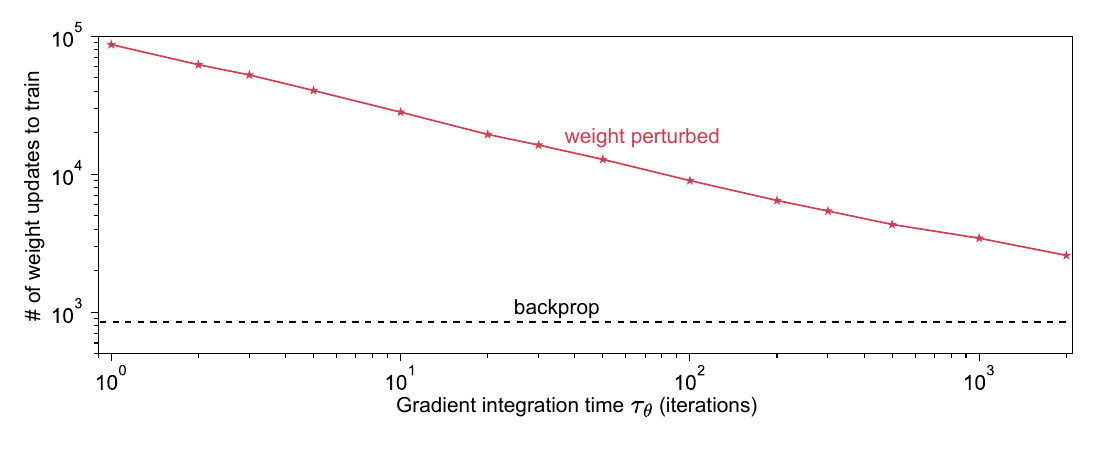}
    \caption{Number of weight updates required to train FashionMNIST on a $2.55\times10^5$ parameter network to 80\% testing accuracy versus \tautheta.
}
    \label{fig5-thetascan}
\end{figure}

\reffig{fig5-thetascan} shows the total number of weight updates it takes for the network to be trained to 80\% testing accuracy as a function of the gradient integration time \tautheta. The total number of weight updates required to train the network can be seen to decrease as a function of integration time, approaching the number of updates required for backpropagation. This is due to the accuracy of \Gmgd increasing with \tautheta. As can be seen from the figure, the number of weight updates can be reduced by several orders of magnitude using this technique. However, in the absence of weight update constraints, setting $\tautheta=1$ still results in the fastest overall training time (as would be measured in terms of wall-clock time on a real piece of hardware). In general, the MGD framework has the flexibility to match the details of the optimization algorithm to the hardware constraints.

A second example of the flexibility of the MGD platform is the implementation of different types of perturbations, including both weight and node perturbations. Any set of mean-zero orthogonal perturbations can be used to yield equivalent results \cite{mgd_paper, Dembo1990}. However, the effect of bias or non-orthogonal perturbations has not yet been explored fully. Node perturbation provides modest improvements in training time over node perturbation. However, implementing these two different training approaches on hardware presents distinct challenges.  To implement node perturbation in the form described in this paper, each synapse requires a circuit capable of performing one multiplication operation to compute the product of the global change in cost $\Loss$ and the perturbation $\dti{w}$. The algorithm can in principle be simplified further to implement only additions or boolean multiplications at the synapse. In the case that a longer integration time $\tautheta > 1$ is desired to reduce the number of weight updates, an additional memory element is required per synapse to store intermediate gradient values. 

Node perturbation requires single-layer backward data transfer, which reintroduces some of the same challenges that we face when trying to implement backprop in a hardware. For example, to compute the gradient estimate (\Gmgd) in node perturbation, each neuron must have circuitry capable of performing $2N$ multiplication operations (the cost change $\dC$ multiplied by the perturbation $\dti{k}$ multiplied by the synapse input $x_j$ ($\dC\dti{k}x_j$)). Each neuron must also store inputs to a layer in a memory buffer, hence requiring larger local memory on hardware. Another requirement of node perturbation is that weights must be linear to facilitate the backpropagation of the multiply-accumulate (MAC) process. Conversely, weight perturbation makes no assumptions about the network architecture, including linearity of synaptic operations. Weight perturbation does however require one perturbation per weight, while node perturbation only requires one perturbation per node.  This may offer significant advantages in hardware with limited write cycles and a high cost to perturbations. Overall, the choice between these two approaches depends on the specific hardware constraints and requirements.

Finally, additional hardware or architecture specific tricks beyond those described in this paper can be used to reduce the number of parameters to be perturbed in a network. For example, an approach combining perturbative training with the tensor train method \cite{Zhao2023} was proposed recently.

\section{Optimizers}

Modern problems require more than just basic stochastic gradient descent due to its relatively slow convergence rate. A wide variety of specialized optimizers and other tricks are used to improve the training process as networks become larger and deeper.  Some common examples of specialized optimizers are Momentum, AdaGrad, and Adam. Other `tricks’ for training large and deep networks include well-designed network initializations, dropout, local competition strategies, and architecture changes such as skip-layer connections.  Given the ubiquity of these tricks in modern machine learning, it would be unfair to compare the performance of basic MGD with the performance of backpropagation with decades worth of optimization built on it.

However, since MGD can be used to compute the gradient to an accuracy close to that of backpropagation, all of these standard tricks may still be applicable and directly translatable into the framework, although the hardware requirements will inevitably increase. In particular, we expect that Momentum, Adam, and Dropout will require additional local memory and multiplications to implement.  Other techniques such as batch normalization are likely to prove significantly more challenging to implement due to the requirement that many network parameters throughout the network be accessed to compute the normalization required for a single parameter update during training. \reffig{fig6-etascan} is a preliminary result demonstrating that MGD can take advantage of more modern optimizers.

\begin{figure}[H] 
    \centering
    \includegraphics[width=\textwidth]{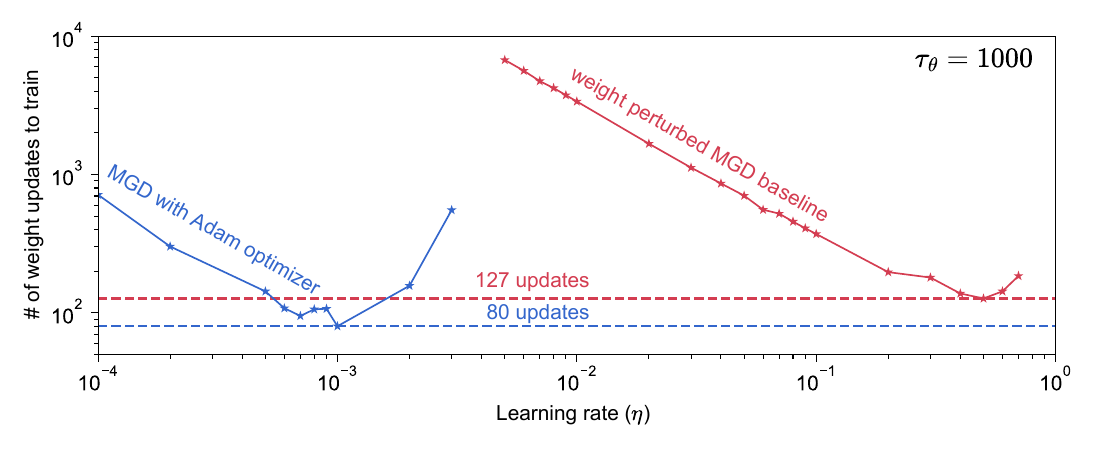}
    \caption{Number of weight updates required to train FashionMNIST to 80\% accuracy on a 255,000-parameter network using MGD with vanilla gradient descent (red) and with the Adam optimizer (blue).}
    \label{fig6-etascan}
\end{figure}

The red line shows the number of updates it takes for the network to reach 80\% testing accuracy using weight perturbation and vanilla stochastic gradient descent. The update rule for vanilla stochastic gradient descent is
\begin{equation}
\label{eq:vanillaupdate}
\W_t = \W_{t-1} - \eta \Gmgd.
\end{equation}
where $\W$ represents the network weights, \Gmgd is the MGD gradient and $\eta$ is the learning rate. We then implemented the Adam optimizer update rule instead of vanilla gradient descent as shown in Eq. \ref{eq:adam}. (Arithmetic operations on vectors are interpreted component-wise.)
\begin{subequations}
\label{eq:adam}
\begin{align}
    \mm_t &= \beta_1 \mm_{t-1} + (1-\beta_1)\Gmgd \\
    \vv_t &= \beta_2 \vv_{t-1} + (1-\beta_2) \Gmgd^2 \\
    \W_t &= \W_{t-1}-\eta\Big(\frac{\mm_t}{\sqrt{\vv_t} + \epsilon}\Big)
\end{align}
\end{subequations}
The vectors $\mm_t$ and $\vv_t$ represent exponentially decaying moving averages of the gradient and its square, while $\beta_1$, $\beta_2$ and $\epsilon$ are user-selected hyperparameters. Following standard guidance we set $\beta_1 = 0.9$, $\beta_2 = 0.999$ and $\epsilon=10^{-8}$ \cite{kingma2014}. To avoid any numerical instabilities in the simulation, the learning rate is set to $\eta=0$ for the initial 1000 iterations.

The simulations were run using the MGD gradient estimate (with $\tautheta = 1000$) within the Adam update rule. The pink line in \reffig{fig6-etascan} shows the number of updates it takes for the network to reach 80\% testing accuracy as a function of the learning rate. We find that there is a 37\% decrease in the number of updates required when using the Adam optimizer. This result demonstrates that MGD can serve as a drop-in replacement for the gradient computed by backpropagation in modern optimization algorithms. This opens the possibility for significantly improving the speed of MGD on different hardware platforms. Other results in the literature indicate similar advantages to augmenting perturbative gradient descent with modern machine learning.  For example, in \refcite{Ren2023}, it was shown that adding greedy local learning significantly improved the training speed for a network of 13 million parameters classifying MNIST and CIFAR10.

\section{Discussion}
The unexpected effectiveness of MGD fits into a growing body of literature showing that an accurate gradient estimate is not required for training neural networks \cite{Lillicrap2016}. The scaling results are consistent with current theoretical understanding of how network size affects training difficulty. This can be seen in investigations of architectures and learning algorithms based on hidden layers with fixed parameters, such as reservoir computers \cite{Yan2024} and extreme learning machines \cite{Huang2006}. As layers become wide enough, training of the inner layers becomes unnecessary altogether, and high accuracy can be obtained with only a linear solve on the output layer (a much faster and simpler algorithm). This paper therefore adds to the body of literature indicating that effective training of neural networks can been done with surprisingly simple algorithms.

Perturbative learning was implemented in CMOS hardware in the 1990s \cite{Matsumoto1990,Cauwenberghs1992, Alspector1992, Kirk1992, Maeda1995, Cauwenberghs1996,Moerland1996, Montalvo1997,Miyao1997, Draghici2000}, but hardware networks only reached a scale of $\sim 10$ weights before the community as a whole shifted towards digital communication and address-event representation \cite{Mead2020}. In the past decade, there has been a resurgence of interest in analog and analog digital neuromorphic hardware. This includes research into novel memristive, optical, and superconducting implementations, among others. Training with conventional backpropagation based algorithms has proved challenging for these types of analog hardware. Perturbative algorithms have been gaining in popularity within this community \cite{Adhikari2015, Wang2018, Bandyopadhyay2022} due to simplicity of implementation. MGD provides a framework for evaluating the training speed of these perturbative algorithms given particular hardware constraints. We hope showing the scalability of these algorithms will prove useful to the development of new hardware.

It has been speculated that perturbative learning takes place in the brain \cite{Seung2003, Fiete2006}, but there is skepticism that it can play a major role due to the poor scaling of the time to estimate the gradient \cite{Lillicrap2020} and results based on training of linear networks \cite{Werfel2005}. The result of this paper -- that the time to train does not scale directly with network size -- will hopefully prompt some re-examination of this conclusion. In this context, another important feature of the MGD framework is that weight updates require only information that is spatially local to the synapse, and a single globally-broadcast cost-change signal.   Unlike backpropagation, signal propagation in the MGD framework is always in the forward direction, the derivative of the activation function is not required to be known or used, and there are no separate forward and backward phases in the learning algorithm. Additionally, small scale demonstrations of training of spiking networks by these techniques were implemented in the early 2000s \cite{Seung2003, Fiete2006}, and have undergone a recent resurgence in interest \cite{Mukhoty2023, Xiao2024}. This offers up the intriguing possibility that this is a general training technique that can apply equally to spiking and non-spiking networks.  Further work on this is required to reach conclusions about spiking networks. 

In this paper we showed modest improvements to the scaling of the time to train a neural network by using node perturbation instead of weight perturbation. The choice of weight or node will likely be influenced strongly by the specifics of the hardware implementation. We also note as a qualitative observation that we found node perturbation to be more challenging to optimize than weight perturbation at almost every stage of its implementation. Other recent work on node perturbation has shown that the scaling of node perturbation is worse when recurrent networks are considered \cite{Zuge2023} and there can be issues with stability \cite{Hiratani2022}. In summary, the trade-offs between node and weight perturbation in hardware implementation highlight the need for a careful evaluation of the specific requirements and capabilities of the hardware platform.  This evaluation ensures that the chosen method aligns with the overall goals of efficiency, speed, and resource utilization.

Finally, we note that perturbative techniques have taken off recently in other contexts beyond the training of neuromorphic hardware. Recent works have investigated similar algorithms for fine-tuning of large language models \cite{Malladi2023,Chen2023,Han2024}. This is because less memory is required to implement than is required for backpropagation. Therefore large language models can be fine-tuned on smaller machines than required for the original training. 

\section{Conclusion}
Our investigation into multiplexed gradient descent (MGD) demonstrates its potential as a scalable and efficient training method for neuromorphic hardware. Despite common misconceptions, we have shown that perturbative methods like MGD can scale to large networks. Although the time required for the stochastic variable $\Gmgd$ to accurately estimate the true gradient $\gradCW$ scales with the number of parameters, it is nevertheless possible that the time to reach a fixed accuracy target can be independent of network size. 

Furthermore, our results demonstrate how MGD can be swapped for backpropagation algorithmically, allowing the application of standard gradient descent enhancements like momentum and Adam. This flexibility enables MGD to achieve comparable testing accuracy as backpropagation, even for networks with over one million parameters. Although MGD may appear slower than backpropagation in simulations, the speed on physical hardware depends on the time to execute iterations of the algorithm. MGD may be much simpler and faster to implement than backpropagation on emerging analog hardware, and speeds competitive with backpropagation training on conventional hardware are plausible given our simulation results. Overall, MGD provides a practical and adaptable solution for on-hardware training, capable of leveraging advanced optimizers and handling large-scale networks within reasonable time frames. Our findings underscore the importance of considering both the hardware architecture and the specific needs of the training process when selecting perturbation methods, paving the way for more efficient and scalable neuromorphic computing systems.

\section*{Data availability}
The data used for all the plots in this work are available from the corresponding author on reasonable request.

\section*{Code availability}
The MGD library and code used to perform these simulations are available on \newline \href{https://github.com/bakhromtjk/mgd_scaling}{\texttt{github.com/bakhromtjk/mgd\_scaling}}.

\section*{Acknowledgments}
The U.S. Government is authorized to reproduce and distribute reprints for governmental purposes notwithstanding any copyright annotation thereon. Analysis performed in part on the NIST Enki HPC cluster. This research was funded by NIST (https://ror.org/05xpvk416) and University of Colorado Boulder (https://ror.org/02ttsq026).

\section*{Author contributions}
ANM, SMB conceptualized these experiments. BGO run the simulations. Analysis and interpretation of the data was done by BGO, ANM, SMB, and AD. All authors co-wrote the manuscript.

\section*{Declaration of interests}
The authors declare no competing interests.

\appendix
\section{Mean and Covariance of $\Gmgd$}
In this section we compute the first two moments---i.e., the mean and covariance---of $\Gmgd$ defined in \eqref{def:G}. Note that the diagonal of the covariance matrix can be used to compute the expectation of the squared norm, ${\rm E}\big(||\Gmgd||^2\big)$.

\subsection{Random perturbations}
In the stochastic perturbation model of MGD the perturbations are a collection of random variables,
\[
\{\dtit{i}{t}|\text{ for }i=1,\ldots,\nP\text{ and }t=1,\ldots,\nT\}.
\]
Here $\nP$ is the dimension of the parameter vector and $\nT$ is the total number of time steps. We assume the following
\begin{enumerate}
\item $\dtit{i}{t}$ are ``symmetric'' meaning that $\E{\dtitn{i}{t}{2k+1}}=0$ for $k\ge0$.
\item $\dtit{i}{t}$ can be distinct over $i$, but are identical over $t$.
\item $\dtit{i}{t}$ are independent over $i$ and $t$.
\end{enumerate}
These assumptions are reasonably general and lead to significant simplifications in the following computations. For reference we define the variance and fourth moment for each index
\begin{align*}
\E{\dtitn{i}{t}{2}} &= \E{\dtitn{i}{t'}{2}} = \s{i}^2\\
\E{\dtitn{i}{t}{4}} &= \E{\dtitn{i}{t'}{4}} = \a{i}^4
\end{align*}
Note these quantities depend on index but not time.

\subsection{Moments of $\Gmgd$}
We assume that magnitudes of the perturbations are sufficiently small that the cost function can be approximated by its linear expansion \eqref{eqn:linearCost}. Thus, for a perturbation vector at time $t$, the change in cost is given by
\[
\dC(t) = \sum_{i} \dCi{i}\dtit{i}{t}
\]
From \eqref{def:G}, we write component MGD vector as
\[
G_m  = \frac{1}{\nT\s{m}^2}\sum_{t}\sum_{i} \dCi{i}\dtit{i}{t}\dtit{m}{t}.
\]
We prove the following theorem

\begin{theorem}\label{thm:MGD}
The MGD random variable, $\Gmgd$, is an unbiased estimator of $\gradC$
\begin{equation}
\label{eqn:meanG}
{\rm E}\big(\Gmgd\big) = \gradC.
\end{equation}
The covariance matrix is given by
\begin{equation}
\label{eqn:covG}
\cov\big(\Gmgd\big)_{m,n} = 
\begin{cases}
\displaystyle
\frac{1}{\nT}\left(
      \left(\dCi{m}\right)^2\left(\frac{\a{m}^4}{\s{m}^4}-1\right)
    + \sum_{\mu \ne m} \left(\dCi{\mu}\right)^2\frac{\s{\mu}^2}{\s{m}^2}
  \right), & m=n \\
\displaystyle
\frac{1}{\nT}\dCi{m}\dCi{n}, &m \ne n
\end{cases}
\end{equation}
\end{theorem}

\begin{proof}
The mean of $G_m$ is computed
\begin{align*}
{\rm E}(G_m)
&= \E{\frac{1}{\nT\s{m}^2}\sum_t\sum_{i} \dCi{i}\dtit{i}{t}\dtit{m}{t}}\\
&= \frac{1}{\nT\s{m}^2}\sum_t\sum_{i} \dCi{i}\E{\dtit{i}{t}\dtit{m}{t}}\\
&= \frac{1}{\nT\s{m}^2}\sum_t\dCi{m}\s{m}^2\\
&= \dCi{m}.
\end{align*}
The second equality follows from linearity of expectation. In going from second to third line, the independence of $\dtit{i}{t}$ and $\dtit{m}{t}$ implies that the expectation of the product is zero for $i \ne m$, and thus these terms in the summation over $i$ drop out. The only non-zero contribution comes from $i=m$, in which case the expected product is $\s{m}^2$. Finally, as $\s{m}^2$ does not depend on $t$ and there are $\nT$ terms in the sum over $t$, the result follows. Thus, we have \refeq{eqn:meanG}.

We next compute the covariance matrix
\begin{align}
\cov\big(\Gmgd\big)
    &= {\rm E}\big(\Gmgd^T\Gmgd\big)-{\rm E}\big(\Gmgd\big)^T{\rm E}\big(\Gmgd\big) \nonumber\\
    &= {\rm E}\big(\Gmgd^T\Gmgd\big)-\gradC^T\gradC.
\label{def:covG}
\end{align}
First we compute the off-diagonal terms,
\begin{align*}
{\rm E}\big(G_m G_n\big)
    &=\E{\frac{1}{\nT^2\s{m}^2\s{n}^2}
        \sum_t\sum_{i} \dCi{i}\dtit{i}{t}\dtit{m}{t}
        \sum_{t'}\sum_{j} \dCi{j}\dtit{j}{t'}\dtit{n}{t'}}\\
    &=\frac{1}{\nT^2\s{m}^2\s{n}^2}
        \sum_{t,t'}\sum_{i,j} \dCi{i}\dCi{j}
        \E{\dtit{i}{t}\dtit{m}{t}\dtit{j}{t'}\dtit{n}{t'}}
\end{align*}
Split the last sum into two cases: $t=t'$ and $t \ne t'$. In either case, any non-zero expectations will be constant as a function of $t$ resulting in $\nT$ summands in the first case and $\nT(\nT-1)$ in the second. Thus we have
\small
\begin{multline*}
    \sum_{t,t'}\sum_{i,j} \dCi{i}\dCi{j}
        \E{\dtit{i}{t}\dtit{m}{t}\dtit{j}{t'}\dtit{n}{t'}} = \\
\nT\sum_{i,j} \dCi{i}\dCi{j}
        \E{\dti{i}\dti{m}\dti{j}\dti{n}} +
\nT(\nT-1) \sum_{i,j} \dCi{i}\dCi{j}
        \E{\dti{i}\dti{m}}\E{\dti{j}\dti{n}}
\end{multline*}
\normalsize
Observe that the first sum over $i,j$ has only two non-zero contributions---$i=m$ and $j=n$, or $i=n$ and $j=m$---and the expectation is $\s{m}^2\s{n}^2$ in both cases. Similarly, the second sum can only contribute when $i=m$ and $j=n$. The result is
\begin{align}
{\rm E}\big(G_m G_n\big)
 &=\frac{1}{\nT^2\s{m}^2\s{n}^2}
    \left( 2\nT\dCi{m}\dCi{n}\s{m}^2\s{n}^2 + \nT(\nT-1)\dCi{m}\dCi{n}\s{m}^2\s{n}^2\right) \nonumber\\
 &= \dCi{m}\dCi{n} + \frac{1}{\nT}\dCi{m}\dCi{n}
 \label{eqn:covGmn}
\end{align}
Note that the first term on the right hand side is the $m,n$ component of $\gradC^T\gradC$. Subtracting this from both sides results in the $m\ne n$ case in \refeq{eqn:covG}

The computation for the diagonal term is slightly more complicated.
\begin{align*}
{\rm E}\big(G_m^2\big)
    &=\E{\frac{1}{\nT^2\s{m}^4}
        \sum_t\sum_{i} \dCi{i}\dtit{i}{t}\dtit{m}{t}
        \sum_{t'}\sum_{j} \dCi{j}\dtit{j}{t'}\dtit{m}{t'}}\\
    &=\frac{1}{\nT^2\s{m}^4}
        \sum_{t,t'}\sum_{i,j} \dCi{i}\dCi{j}
        \E{\dtit{i}{t}\dtit{m}{t}\dtit{j}{t'}\dtit{m}{t'}}
\end{align*}
Again, we split the sum over $t,t'$ into a diagonal and off-diagonal part
\small
\begin{multline*}
\sum_{t,t'}\sum_{i,j} \dCi{i}\dCi{j}
        \E{\dtit{i}{t}\dtit{m}{t}\dtit{j}{t'}\dtit{m}{t'}} = \\
\nT\sum_{i,j} \dCi{i}\dCi{j}
        \E{\dti{i}\dti{j}\dtin{m}{2}} +
\nT(\nT-1) \sum_{i,j} \dCi{i}\dCi{j}
        \E{\dti{i}\dti{m}}\E{\dti{j}\dti{m}}
\end{multline*}
\normalsize
As before, the second sum has a non-zero contribution only for $i=m$ and $j=m$ in which case the expectations result in $\s{m}^4$. The first sum has a second-moment contributions when $i=j=\mu \ne m$ and a single fourth-moment when $i=j=m$
\[
\sum_{i,j} \dCi{i}\dCi{j}
        \E{\dti{i}\dti{j}\dtin{m}{2}} = 
\sum_{\mu \ne m} \dCi{\mu}^2\s{\mu}^2\s{m}^2 +  \dCi{m}^2\a{m}^4
\]
Using these simplifications we have
\begin{align}
{\rm E}\big(G_m^2\big)
 &=\frac{1}{\nT^2\s{m}^4}\left( 
    \nT\left(\sum_{\mu \ne m} \left(\dCi{\mu}\right)^2\s{\mu}^2\s{m}^2 + \left(\dCi{m}\right)^2\a{m}^4\right) + 
    \nT(\nT-1)\left(\dCi{m}\right)^2\s{m}^4 \right) \nonumber\\
  &= \left(\dCi{m}\right)^2 
  + \frac{1}{\nT}\left(
      \left(\dCi{m}\right)^2\left(\frac{\a{m}^4}{\s{m}^4}-1\right)
    + \sum_{\mu \ne m} \left(\dCi{\mu}\right)^2\frac{\s{\mu}^2}{\s{m}^2}
  \right) \label{eqn:covGmm}
\end{align}
Again, subtracting $(\partial \Loss/\partial \theta_m)^2$ from both sides results in the diagonal case of \refeq{eqn:covG}.
\end{proof}

Note that the diagonal terms of the covariance matrix can be reorganized and summed resulting in the following corollary
\begin{corollary}\label{thm:normG} The norm of the MGD random variable $\|\Gmgd\|^2$ is a biased estimator of $\|\gradC\|^2$
\begin{equation}
\label{eqn:Gnorm}
\E{||\Gmgd||^2} = ||\gradC||^2 +
\frac{1}{\nT}\sum_{m=1}^{\nP}
\left(\left(\dCi{m}\right)^2\left(\frac{\a{m}^4}{\s{m}^4}-1\right)
    + \sum_{\mu \ne m} \left(\dCi{\mu}\right)^2\frac{\s{\mu}^2}{\s{m}^2}\right)
\end{equation}
\end{corollary}
\begin{proof}
As $\Gmgd$ is an unbiased estimator of $\gradC$ (see \eqref{eqn:meanG}), from the diagonal terms of \eqref{def:covG} we see that
\[
{\rm E}\big(G_m^2) = \left(\dCi{m}\right)^2 + \cov\big(\Gmgd\big)_{m,m}
\]
Summing over components $m$ gives \refeq{eqn:Gnorm}.
\end{proof}

Finally, assuming that $\dtit{m}{t}$ are independent, identically distributed Bernoulli random variables taking values $\pm \epsilon$ with probability $1/2$ simplifies these expressions. In this case, the variance and forth moments are
\[
\dtit{m}{t} \sim {\rm Bernoulli}\left(\pm \epsilon,\frac{1}{2}\right)
    \Longrightarrow \s{m}=\a{m}=\epsilon.
\]
Therefore, the first term in the sum \eqref{eqn:Gnorm} vanishes. Furthermore, as the ratio of variances in the inner sum becomes one, we evaluate this sum as 
\[
\sum_{\mu \ne m} \left(\dCi{\mu}\right)^2 = ||\gradC||^2 - \left(\dCi{m}\right)^2.
\]
Thus, for the Bernoulli perturbation model we have
\begin{equation}
\label{eqn:GnormBias}
    \E{||\Gmgd||^2} = \left(1 + \frac{\nP-1}{\nT}\right)||\gradC||^2
\end{equation}
It follows that the norm of the MGD-derived gradient vector is biased with respect to the true norm of the gradient, and furthermore this norm will be a relatively poor estimator unless $\nT\gg\nP$.

\section{MGD Pseudocode}
\label{sec:mgdCode}
\begin{table}[h]
\centering
\begin{tabular}{|c|l|}
\hline
Symbol & Description \\
\hline

$x$     & input to the network   \\
$y$ & target output from the network   \\
$\yhat$     & inferred output from the network   \\
$N$ & number of trainable parameters\\
$K$ & number of perturbed parameters\\
$\Theta$     & trainable network parameters   \\
$\dtheta$     & perturbations applied to $\Theta$ \\
$\delta = |\dti{k}|$ & amplitude of perturbations applied to $\Theta$ \\
$\tautheta$     & the period with which $\Theta$ is updated \\
$\eta$     & learning rate \\

$\Loss_0$     & baseline cost before perturbing   \\
$\Loss$     & final cost after perturbing   \\
$\Delta \Loss_l$     & difference in cost due to perturbation   \\
$\Gmgd$ & estimated gradient \\
$\Gamma=\delta^2\sqrt{K}\sqrt{1+(K-1)/\tau_\theta}$ & a normalization factor (see \refeq{eqn:GnormBias}). \\ 

$n$     & iteration step number   \\
$\textrm{num\_iterations}$     & total number of iterations   \\
$\textrm{num\_layers}$     & total number of layers   \\
$\textrm{weight\_perturbed}$    & Boolean; is the network trained using weight perturbation   \\
$\textrm{node\_perturbed}$    & Boolean; is the network trained using node perturbation   \\
\hline
\end{tabular}
\caption{Description of the variables used in the algorithm}
\label{symbols}
\end{table}

\begin{algorithm}
\caption{The weight and node-perturbed MGD algorithm.}
    \begin{algorithmic}[1]
          \State Initialize parameters $\theta$
          \For{$n$ \textbf{in} $\textrm{num\_iterations}$}            
              \State Input new training sample $x$, $y$
            \If{($n \bmod \tautheta$ = 0)}
              \State Set perturbations to zero $\theta \leftarrow 0$
              \If{$\textrm{node\_perturbed}$}
              \State Compute input to every layer $\xhat_l\leftarrow\yhat_{l-1}$ 
               \EndIf         

              \State Update baseline cost $\Loss_0 \leftarrow C(f(x;\Theta), y)$
              
            \EndIf
                \State Update perturbations $\theta$
            \For{$l$ \textbf{in} $\textrm{num\_layers}$}            

            \State Compute output 
            \State $\yhat_l \leftarrow f(x; \Theta+\theta_l)$
            \State Compute cost $\Loss_l \leftarrow C(\yhat_l ,y_l)$
            \State Compute change in cost $\Delta \Loss_l \leftarrow \Loss_l-\Loss_0$
            \State Compute error signal 

            \If{$\textrm{weight\_perturbed}$}
            \State $e_l \leftarrow \Delta \Loss_l \theta_l / \Gamma$
            \EndIf

            \If{$\textrm{node\_perturbed}$}
            \State $e_l \leftarrow \Delta \Loss_l \theta_l x_l / \Gamma$
            \EndIf
            \State
            \State Accumulate gradient approximation $G_l \leftarrow G_l + e_l$            
            \EndFor

            \If{($n \bmod \tautheta$ = 0)}
              \State Update parameters $\Theta \leftarrow \Theta - \eta G$
                \If{($n \bmod \tautheta \neq \infty$)}    
                  \State Reset gradient approximation $G \leftarrow 0$
                \EndIf
            \EndIf
          \EndFor
    \end{algorithmic}
    \label{mgdalgorithm}
\end{algorithm}

\newpage
\printbibliography

\end{document}